\documentclass[runningheads]{llncs}
\usepackage[T1]{fontenc}
\usepackage{graphicx}
\usepackage{amsmath}
\usepackage{xspace}
\usepackage{bbm}
\usepackage{tikz}
\usetikzlibrary{shapes.geometric, arrows.meta, positioning}
\usepackage{multirow}
\usepackage{amssymb}
\usepackage[utf8]{inputenc}
\usepackage{tabularx}
\usepackage{booktabs}
\usepackage[algo2e,ruled,vlined,linesnumbered]{algorithm2e}
\usepackage{xcolor}

\usepackage{siunitx}
\usepackage{subcaption}
\usepackage{placeins}
\usepackage{bm}
\usepackage[normalem]{ulem}

\newcommand{\best}[2]{\uline{$\mathbf{#1}$ {\scriptsize$\mathbf{\pm\,#2}$}}}

\newcommand{\bestsolver}[2]{%
  \uline{$\mathbf{#1}$}%
  \phantom{{\scriptsize$\mathbf{\pm\,#2}$}}%
}

\newcommand{\boldscore}[2]{$\mathbf{#1}$ {\scriptsize$\mathbf{\pm\,#2}$}}

\newcommand{\oea}{\mbox{${(1 + 1)}$~EA}\xspace}

\newcommand{\LO}{\textsc{Leading\-Ones}\xspace}
\newcommand{\leadingones}{\LO}
\newcommand{\DLB}{\textsc{DLB}\xspace}
\newcommand{\deceptiveleadingblocks}{\textsc{Deceptive\-Leading\-Blocks}\xspace}

\newcommand{\cliff}{\textsc{Cliff}\xspace}

\newcommand{\jump}{\textsc{Jump}\xspace}

\newcommand{\R}{\ensuremath{\mathbb{R}}}

\newcommand{\N}{\ensuremath{\mathbb{N}}} 

\newcommand{\probLB}{p_{\mathrm{LB}}}

\newcommand{\rmut}{r_{\mathrm{mut}}}
\newcommand{\rc}{r_{\mathrm{c}}}
\newcommand{\nn}{n} 
\newcommand{\Eint}{E_{\mathrm{int}}}
\newcommand{\phiMIS}{\phi_{\mathrm{MIS}}}
\newcommand{\phiMC}{\phi_{\mathrm{MC}}}
\newcommand{\FMC}{F_{\mathrm{MC}}}
\newcommand{\FMIS}{F_{\mathrm{MIS}}}

\newcommand{\assign}{\leftarrow}

\let\originalleft\left
\let\originalright\right
\renewcommand{\left}{\mathopen{}\mathclose\bgroup\originalleft}
\renewcommand{\right}{\aftergroup\egroup\originalright}

\SetKw{KwWP}{with prob.}
\SetKw{KwElse}{else}

\begin{document}
{\sloppy
\title{A Fresh Look at Lamarckian Evolution and the Baldwin Effect}


\author{Inès Benito\inst{1}\orcidID{0009-0005-2435-2152} \and
Johannes F. Lutzeyer\inst{1}\orcidID{0000-0001-9512-7917} \and
Benjamin Doerr\inst{1}\orcidID{0000-0002-9786-220X}}
%

\institute{Laboratoire d'Informatique (LIX), CNRS, École Polytechnique, Institut Polytechnique de Paris, France\\
\email{\{firstname.lastname\}@polytechnique.edu}}
\maketitle              
\begin{abstract}
Baldwinian and Lamarckian evolution have existed for a long time in evolutionary algorithms (EAs) without ever dominating the academic literature or practical applications. In this work, we use modern empirical and theoretical methods to revisit Lamarckian and Baldwinian evolution and rigorously compare them with the generic Darwinian evolution. On the empirical side, we run a comprehensive suite of experiments on graphs from six different datasets from the recent GraphBench benchmark on Maximum Independent Set and Maximum Cut problems. Our results show that Baldwinian and Lamarckian evolution consistently outperform Darwinian evolution, confirming the great potential of local search augmented evolutionary algorithms. Notably, in the great majority of cases, all EAs outperform recent deep learning baselines and approach the performance of highly specialised heuristic and exact solvers. We furthermore report a high-performing set of generalist parameters for all studied evolution types that we hope will be of use to practitioners in future. 
On the theoretical side, we extend the existing DeceptiveLeadingBlocks benchmark to arbitrary block length~$k$. For all constant $k$, we then prove asymptotically tight runtime bounds for the $(1+1)$ EA in the three evolution types on this benchmark. For Baldwinian evolution, these are independent of~$k$, whereas for the other two evolution types, the runtimes steeply increase with growing value of~$k$.


\end{abstract}

\section{Introduction}
\label{sec:introduction}

Evolutionary algorithms (EAs) are population-based metaheuristics inspired by the principles of natural selection~\cite{eiben2015}. Starting with a population of candidate solutions, they iteratively apply evolutionary operators (crossover, mutation and selection) to retain the fittest individuals for the next generation. Unlike random search, selection pressure biases the search towards promising regions of the solution space, enabling EAs to scale to large and complex problems.

In some cases, standard EAs are effective at exploration, rapidly identifying promising regions of the search space, but can be weaker at exploitation~\cite{eiben2015}. EAs augmented with a local search step have been proposed to address this imbalance, giving rise to memetic algorithms~\cite{moscato1989}. Two distinct paradigms have emerged. In Lamarckian evolution, each offspring is replaced by its locally improved version before selection~\cite{whitley1994}. In Baldwinian evolution~\cite{baldwin1896,hinton1987}, the offspring is left unchanged but evaluated on the fitness of the solution returned by local search. A third approach, also known as ``partial Lamarckianism''~\cite{houck1997}, combines Lamarckian and Baldwinian evolution, which we refer to as LB evolution. In LB evolution, the locally improved solution replaces the offspring with probability $\probLB$. In practice, most work on this topic has favoured either a pure Lamarckian approach or LB evolution~\cite{eiben2015}. Empirical studies have highlighted a trade-off between Lamarckian and Baldwinian evolution: the former tends to converge faster, while the latter is more reliable on harder instances~\cite{whitley1994,castillo2006}.

Despite their long history, Lamarckian and Baldwinian evolution have not recently been empirically evaluated on state-of-the-art benchmarks at scale. Moreover, to our knowledge, there are no runtime analyses of Baldwinian EAs and only sporadic ones of Lamarckian EAs under the name of memetic algorithms. 

In this work, we take a fresh look at all four evolution types using both large-scale empirical evaluation and modern theoretical tools. On the empirical side, we conduct a comprehensive suite of experiments on the GraphBench benchmark~\cite{stoll2026graphbench}, covering small and large graphs from three random graph models on the Maximum Independent Set and Maximum Cut problems. Our results show that Baldwinian evolution outperforms Lamarckian evolution on the Maximum Independent Set problem, while both substantially outperform deep learning baselines and match or approach the performance of highly specialised solvers. In our experiments, Darwinian evolution is consistently the weakest, demonstrating the great potential of local search augmentation in practice. We also report a robust set of generalist parameters that incurs an average solution quality loss below $1\%$ compared to problem-specific tuning for all local-search-based evolution types, providing a set of reference values for practitioners. The complete source code reproducing our experiments is publicly available\footnote{\url{https://github.com/Hypatia-II/polerina}}. 

On the theoretical side, we extend the \deceptiveleadingblocks (\DLB) benchmark~\cite{LehreN19foga} from block length $k=2$ to arbitrary $k \ge 2$ and prove asymptotic upper and lower bounds on the runtime of Darwinian, Lamarckian, and Baldwinian variants of the \oea. In terms of iterations, Baldwinian evolution is asymptotically faster than Lamarckian, which is in turn asymptotically faster than Darwinian evolution for all $k \ge 2$. When accounting for the cost of the local search procedure, the precise ordering depends on the implementation, but Baldwinian evolution outperforms the others from a small value of $k$ onwards, aligning with the empirical advantage of Baldwinian evolution on hard instances.

\section{Related Work}
\label{sec:relatedwork}

We now provide an overview of the existing literature on Baldwinian, Lamarckian and LB evolution. Hinton and Nowlan \cite{hinton1987} showed in the context of Neural networks that Baldwinian evolution modifies the search space by smoothing the fitness landscape, thereby creating a basin of attraction around the global optimum. Whitley, Gordon and Mathias~\cite{whitley1994} also found that Baldwinian evolution allows for accelerated convergence towards a global optimum. 

Comparative studies showed that while Lamarckian evolution tends to converge faster on easier problems, it is more susceptible to getting trapped in local optima. Baldwinian evolution, on the other hand, is often slower but more reliably converges to the global optimum for difficult problems \cite{whitley1994,castillo2006}. Lamarckian evolution has also been applied in other domains, such as evolutionary robotics \cite{LuoMTE23,luo2025}. A detailed discussion of the encoding requirements for Lamarckian evolution is provided in Appendix~\ref{app:lamarck}.

Houck et al. \cite{houck1997} studied LB evolution on bounded nonlinear and combinatorial optimisation problems, arguing it provides the most robust performance across diverse problem sets. They find, however, that $\probLB$ is problem-specific.

Given the little understanding of the relation of Darwinian, Lamarckian, and Baldwinian evolution in evolutionary computation in general, it is not surprising that there are very few theoretical works on this topic. The only work regarding the Baldwin effect which we are aware of is the seminal paper by Whitley, Gordon, and Mathias~\cite{whitley1994}. It defines a moderately deceptive function on bit-strings of length four and studies how a simple evolutionary algorithm in the infinite population model copes with this problem. The mathematical study is restricted to the analysis of how the fitness landscape changes, which determines the change of the frequencies of the 16 individuals in the population in one iteration. How this affects the global behaviour of the algorithm is then numerically evaluated. The landscape analysis exhibits several interesting properties of Lamarckian and Baldwinian evolution which still shape today's understanding of these two types of evolution. With a problem size of four and the infinite population model, however, the results are hard to transfer into the asymptotic runtime analysis perspective, which dominates the theoretical analysis of EAs today. We note that in \cite{whitley1994} in both the Lamarckian and the Baldwinian EA the local search procedure performs only one hillclimbing step. 

While the term Lamarckian evolution is rarely used in these works, most memetic algorithms in fact use Lamarckian evolution, that is, they apply a local search improvement heuristic to each solution generated. We refer to the latest of these works \cite{NguyenS20} or the survey \cite{Sudholt11bookchapter} for a detailed discussion. Overall, these works show that local search can often speed up evolutionary algorithms, however, a complete picture of when and how memetic algorithms profit from local search has not yet been obtained. In \cite{NguyenS20}, local search helps greatly in leaving local optima, but the hurdle benchmark studied there has the particularly suitable property that each local optimum has a Hamming neighbour from which local search can lead to a better local optimum. When this property is absent, memetic algorithms using local search can face enormous difficulties, as shown in~\cite{Sudholt11}.

\section{Empirical Analysis}
We now present an empirical comparison of the four evolution types: Darwinian, Baldwinian, Lamarckian, and LB across two combinatorial optimisation problems on graphs. We introduce the two studied problems and their fitness functions in Section~\ref{sec:problems}, we then define the evolutionary algorithms and the local search used in Section~\ref{sec:ga}. Subsequently, we report our experimental setup in Section~\ref{sec:protocol} and provide our main empirical results in Section~\ref{sec:results}.

\subsection{Problem Formulations}
\label{sec:problems}

We study two discrete combinatorial optimisation problems defined on graphs: Maximum Independent Set (MIS) and the Maximum Cut (MC). We focus on their unweighted versions where all edges are assigned unit weights. Let $G = (V, E)$ be an undirected graph, where $V = \{v_1, \dots, v_n\}$ is the set of vertices and $E \subseteq \{ \{v_i, v_j\} : v_i, v_j \in V\}$ is the set of edges. Let $A$ be the adjacency matrix of $G$, where $A_{ij} = 1$ if $\{v_i, v_j\} \in E$ and $0$ otherwise. A candidate solution is represented by a binary vector $x \in \{0, 1\}^n$, where $x_i = 1$ if vertex $v_i$ is selected and $x_i = 0$ otherwise. The MIS problem seeks the largest subset of vertices such that no two vertices in the subset are connected by an edge. The MC problem seeks a partition of the vertex set that maximises the number of edges between the two parts. Formal problem formulations are given in Appendix~\ref{app:problems}. We now proceed to define problem specific fitness functions.

For the MIS problem, we allow non-independent sets in the population and use a penalised fitness function to handle infeasible individuals. This relaxation enables the EAs to explore infeasible regions of the search space, potentially facilitating the crossing of fitness valleys. Infeasible individuals are penalised proportionally to the number of independence violations they incur, steering the search towards feasible solutions. To determine the total number of violations, we evaluate the number of internal edges within the subgraph induced by $x$, which is computed as $\Eint(x) = \frac{1}{2} x^\top A x$. A set is independent if and only if $\Eint(x) = 0$. The EA fitness $\FMIS(x)$ is hence defined as
\begin{equation}
    \FMIS(x) =
    \begin{cases}
      \sum_{i=1}^\nn x_i, & \text{if } \Eint(x) = 0; \\
      - \Eint(x), & \text{otherwise}.
    \end{cases}
\end{equation}

For the MC problem, since every binary vector $x \in \{0,1\}^\nn$ defines a valid partition of $V$, there are no constraints to penalise. The EA fitness $\FMC(x)$ is therefore simply formulated as
\begin{equation}
\FMC(x) = (1 - x)^\top A x.
\end{equation}

\subsection{Evolutionary Algorithms and Local Search}
\label{sec:ga}
To tackle the MIS and MC problems, we work with evolutionary algorithms with different evolution types. The core of our implementation is a $(\mu + \lambda)$ Evolutionary Algorithm which maintains a population $P$ of fixed size $\mu$, initialised by sampling from $\{0,1\}^n$. For every generation, the $\lambda$ offspring $y_i$ are generated independently: with probability $\rc$, offspring $y_i$ results from a uniform crossover between two randomly selected parents, followed by mutation; otherwise, with probability $1 - \rc$, offspring results from the mutation of a single randomly sampled parent. The mutation is a standard Bernoulli mutation where each bit $x_i$ is independently flipped with probability $\rmut = 1/\nn$, yielding an expected number of one bit-flip per offspring. For Baldwinian, Lamarckian and LB evolutions, the offspring undergo a local search. In our local search, we employ a steepest ascent and we search until local optimality \cite{eiben2015}. The fitness of the local optimum found is used for evaluation. For Lamarckian evolution, the improved offspring always replaces the initial one, for LB evolution the replacement occurs with probability $\probLB$, and in Baldwinian evolution, the initial offspring is preserved. For Darwinian evolution, the offspring are evaluated directly after mutation. In all EAs, the offspring $y_i$ are evaluated once and their fitness is stored in $f_{y_i}$. Finally, the $\mu$ fittest individuals from the combined pool of parents and offspring are retained for the next generation. The pseudo-code is given in Algorithm~\ref{alg:ga_hybrid}.
\begin{algorithm2e}[t]                                          
   Initialise population $P$ with $\mu$ individuals randomly sampled from $\{0,1\}^n$\;
  \Repeat{forever}{                                                                                                  
      \For{$i \assign 1$ \KwTo $\lambda$}{
          Randomly sample $x_1, x_2$ from $P$\;
          \KwWP $\rc$: $y_i \assign \textsc{Mutate}(\textsc{Crossover}(x_1, x_2))$; \KwElse $y_i \assign \textsc{Mutate}(x_1)$\;
          \lIf{Darwinian}{$f_{y_i} \assign f(y_i)$} \lElse{$y_{i,\mathrm{LS}} \assign \textsc{LocalSearch}(y_i)$}
          \lIf{Baldwinian}{$f_{y_i} \assign f(y_{i,\mathrm{LS}})$}
          \lIf{Lamarckian}{$f_{y_i} \assign f(y_{i,\mathrm{LS}})$;\enspace $y_i \assign y_{i,\mathrm{LS}}$}
          \lIf{LB}{$f_{y_i} \assign f(y_{i,\mathrm{LS}})$;\enspace \KwWP $\probLB$: $y_i \assign y_{i,\mathrm{LS}}$}      }                  
      $P \gets \text{the } \mu \text{ fittest from } P \cup \{y_1, \ldots, y_\lambda\}$\;
  }     
  \caption{$(\mu + \lambda)$ EA with Darwinian, Baldwinian, Lamarckian, or LB evolution to maximise the fitness function $f : \{0,1\}^n \to \R$.}
  \label{alg:ga_hybrid}                                                                                              
  \end{algorithm2e}

\textbf{MIS Local Search.} For the MIS problem, local search acts as a repair mechanism. For a given node $v_i $, we define the independence violation count as the number of its neighbours in the selected set, denoted by $c_i(x) = x_i \sum_{j=1}^n A_{ij} x_j$. The repair strategy iteratively selects the node $v_i$ with the highest $c_i(x)$ and removes it from the set ($x_i \leftarrow 0$). If multiple vertices share the same $c_i(x)$, a random tie-breaking rule is applied. This process repeats until $\Eint(x) = 0$.

\textbf{MC Local Search.} For the Maximum Cut problem, the local search iteratively flips nodes from one partition set to the other. In each step, the vertex that yields the maximum positive gain in cut size is identified. To ensure computational efficiency, we maintain an active list of nodes with positive gains, avoiding a full scan of all nodes at each iteration. When a node is flipped, the gains of its neighbours are updated incrementally and the active list is dynamically adjusted. Random tie-breaking is also used when multiple flips yield the same maximum gain. We stop when no additional gains are possible.

\textbf{Difference Between Repair and Local Search.} The two local search procedures above differ in nature. For MC, the local search optimises the objective (cut size) until no improving flip exists, guaranteeing a local optimum. For MIS, the procedure is a repair mechanism: it only removes vertices to restore feasibility and never attempts to augment the independent set, yielding a feasible but not necessarily locally optimal solution. One could extend this work, as suggested by a reviewer, by following the repair step with a proper local search step for the MIS problem.

\subsection{Experimental Setup}
\label{sec:protocol}
We evaluate the four evolution types on six graph datasets from the GraphBench suite.  We now detail the benchmark and experiments. Implementation details can be found in Appendix~\ref{app:implementation}.

\textbf{Benchmark.} We conducted experiments using the GraphBench benchmarking suite \cite{stoll2026graphbench}. Specifically, we focus on the ``Combinatorial Optimization'' domain, which provides standardised datasets for the MIS and MC problems. GraphBench provides three random graph models: Erdős-Rényi (ER), Barabási-Albert (BA), and RB-Graphs with two size scales (small and large). Small-scale graphs contain $200$-$300$ nodes and large-scale graphs $700$-$1200$ nodes. Full dataset characteristics are available in \cite{stoll2026graphbench}. GraphBench benchmarks supervised and unsupervised learning approaches. Our EAs fall into the unsupervised category, i.e., no ground-truth labels are required to run EAs. As an upper-bound reference, GraphBench provides solver-generated solutions (see Appendix \ref{app:benchmark}). For MIS, solutions are produced by Karlsruhe Maximum Independent Sets (KaMIS) \cite{lamm2019}, a highly specialised heuristic solver. For MC, solutions are obtained by formulating the problem as an integer programming problem and solving it with Gurobi Optimisation, LLC \cite{gurobi2024} with a timeout of 3600 seconds.

\textbf{Experiments.} Since the EAs employ an elitist $(\mu + \lambda)$ selection strategy, the best found solution is always retained in the population. Hence, we report the maximum fitness score (independent set size for MIS or cut set size for MC)  found in the final population obtained with a fixed computational budget of $40{,}000$ fitness evaluations per experiment. 

The EA parameter search space is provided in Table~\ref{tab:parameter_grid} (Appendix~\ref{app:experimental_details}). All four evolution types share the same grid over population size $\mu$, number of offspring $\lambda$, crossover rate $\rc$, initialisation type, and mutation type. The Lamarckian probability $\probLB$ is additionally tuned for LB evolution. This results in a total of $252$ parameter combinations tested. The grid search is conducted on the training split of the $3$ graph kinds (ER, BA, and RB-Graphs), in the two sizes (small and large). For the rest of the paper, and in accordance with GraphBench, we refer to the association of a graph kind and a size as a dataset. To find the optimal parameters, we sample $0.2\%$ of each of the six datasets, yielding $70$ graphs per dataset and $420$ graphs in total. The $252$ parameter combinations are evaluated on all $420$ graphs with one repetition per graph, resulting in $252 \times 420 = 105{,}840$ EA runs per problem. For each (evolution type, dataset) pair, the selected configuration is the one that maximises the mean objective value over the $70$ sampled graphs. This procedure is performed for MIS and MC.

Once the optimal parameters are found for each dataset (6) and evolution type (4), we evaluate the corresponding EAs on the full test set of $7{,}500$ graphs using three independent random seeds, for a total of $7{,}500 \times 6 \times 4 \times 3 = 540{,}000$ EA runs per problem. We report the mean objective value and its mean standard deviation. For each graph the standard deviation is computed across the three seeds; the reported standard deviation is the mean of the $7{,}500$ per-graph values.

\subsection{Results and Analysis}
\label{sec:results}

In this section, we report the results of our experiments and compare them against GraphBench baselines. We analyse solution quality via mean score, i.e., independent set size or cut set size, then examine convergence trajectories and diversity of both optimal solutions found and population over the runs. We also report runtime and present a single cross-problem parameter configuration.

\begin{table}[t]
\centering
\caption{Mean score ($\uparrow$, $\pm$ std) for EA evolution vs.\ GraphBench baselines across datasets for MIS and MC. Underlined bold typesets the best EA and best baseline. Bold typesets the values within standard deviation of the best values.}

\label{tab:performance}
\resizebox{\textwidth}{!}{%
\begin{tabular}{llrrrrrr}
\toprule
\textbf{Problem} & \textbf{Type} & \textbf{ER Small} & \textbf{ER Large} & \textbf{BA Small} & \textbf{BA Large} & \textbf{RB Small} & \textbf{RB Large} \\
\midrule
  \multirow{9}{*}{MIS} & GIN & $25.42$ {\scriptsize $\pm$ $0.41$} & $26.28$ {\scriptsize $\pm$ $0.41$} & $100.16$ {\scriptsize $\pm$ $3.67$} & $135.00$ {\scriptsize $\pm$ $0.72$} & $17.29$ {\scriptsize $\pm$ $0.33$} & $14.00$ {\scriptsize $\pm$ $0.32$} \\
   & GT & $22.98$ {\scriptsize $\pm$ $0.47$} & $24.98$ {\scriptsize $\pm$ $0.29$} & $99.58$ {\scriptsize $\pm$ $6.45$} & $114.26$ {\scriptsize $\pm$ $0.60$} & $16.54$ {\scriptsize $\pm$ $0.48$} & $13.41$ {\scriptsize $\pm$ $0.14$} \\
   & MLP & $23.18$ {\scriptsize $\pm$ $0.02$} & $24.26$ {\scriptsize $\pm$ $0.45$} & $95.11$ {\scriptsize $\pm$ $2.04$} & $114.49$ {\scriptsize $\pm$ $0.76$} & $16.11$ {\scriptsize $\pm$ $0.10$} & $13.04$ {\scriptsize $\pm$ $0.21$} \\
   & DeepSets & $23.05$ {\scriptsize $\pm$ $0.06$} & $24.22$ {\scriptsize $\pm$ $0.06$} & $95.08$ {\scriptsize $\pm$ $0.17$} & $114.89$ {\scriptsize $\pm$ $0.02$} & $16.02$ {\scriptsize $\pm$ $0.03$} & $13.18$ {\scriptsize $\pm$ $0.04$} \\
   & Solver & \best{33.60}{1.43} & \best{45.64}{0.63} & \best{142.86}{16.54} & \best{433.77}{19.17} & \best{20.80}{1.82} & \best{42.55}{4.45} \\
\cmidrule{2-8}
   & Darwinian & $26.47$ {\scriptsize $\pm$ $1.44$} & $30.73$ {\scriptsize $\pm$ $1.54$} & $140.18$ {\scriptsize $\pm$ $1.34$} & $405.72$ {\scriptsize $\pm$ $4.69$} & $16.96$ {\scriptsize $\pm$ $0.75$} & $31.20$ {\scriptsize $\pm$ $1.23$} \\
   & Baldwinian & \best{33.00}{0.30} & \best{42.06}{0.90} & \best{143.42}{0.03} & \best{433.63}{0.32} & \boldscore{20.04}{0.21} & \boldscore{36.89}{0.52} \\
   & Lamarckian & $30.37$ {\scriptsize $\pm$ $0.94$} & $37.79$ {\scriptsize $\pm$ $1.06$} & \boldscore{143.12}{0.33} & $429.02$ {\scriptsize $\pm$ $2.03$} & \boldscore{19.45}{0.47} & \boldscore{37.14}{0.88} \\
   & LB & \boldscore{32.79}{0.40} & \boldscore{40.66}{1.72} & \boldscore{143.38}{0.06} & \boldscore{433.53}{0.36} & \best{20.10}{0.21} & \best{37.61}{0.88} \\
\midrule
  \multirow{9}{*}{MC} & GIN & $2327.9$ {\scriptsize $\pm$ $24.8$} & $20878.0$ {\scriptsize $\pm$ $107.9$} & $397.0$ {\scriptsize $\pm$ $0.6$} & $1044.1$ {\scriptsize $\pm$ $0.6$} & $2106.7$ {\scriptsize $\pm$ $14.6$} & $24748.0$ {\scriptsize $\pm$ $87.8$} \\
   & GT & $2172.7$ {\scriptsize $\pm$ $91.8$} & $16534.0$ {\scriptsize $\pm$ $278.0$} & $363.8$ {\scriptsize $\pm$ $0.6$} & $986.9$ {\scriptsize $\pm$ $3.1$} & $1925.7$ {\scriptsize $\pm$ $32.8$} & $21524.0$ {\scriptsize $\pm$ $184.0$} \\
   & MLP & $1866.7$ {\scriptsize $\pm$ $67.6$} & $7335.4$ {\scriptsize $\pm$ $57.5$} & $308.7$ {\scriptsize $\pm$ $0.2$} & $929.2$ {\scriptsize $\pm$ $4.1$} & $1727.7$ {\scriptsize $\pm$ $165.1$} & $20357.0$ {\scriptsize $\pm$ $249.6$} \\
   & DeepSets & $33.6$ {\scriptsize $\pm$ $20.8$} & $27.7$ {\scriptsize $\pm$ $6.8$} & $1.1$ {\scriptsize $\pm$ $0.8$} & $154.3$ {\scriptsize $\pm$ $151.5$} & $140.0$ {\scriptsize $\pm$ $155.5$} & $3575.9$ {\scriptsize $\pm$ $730.0$} \\
   & Solver & \bestsolver{2835.5}{00.0} & \bestsolver{23884.0}{00.0} & \bestsolver{460.9}{00.0} & \bestsolver{1260.4}{00.0} & \bestsolver{2920.1}{00.0} & \bestsolver{33914.0}{00.0} \\
\cmidrule{2-8}
   & Darwinian & $2861.5$ {\scriptsize $\pm$ $14.1$} & $23548.9$ {\scriptsize $\pm$ $44.3$} & $398.4$ {\scriptsize $\pm$ $4.5$} & $1186.1$ {\scriptsize $\pm$ $8.8$} & $2805.9$ {\scriptsize $\pm$ $12.5$} & $31622.6$ {\scriptsize $\pm$ $297.6$} \\
   & Baldwinian & \boldscore{2910.1}{0.7} & \boldscore{23857.0}{9.1} & \boldscore{412.6}{0.6} & \best{1248.1}{2.7} & \boldscore{2822.1}{0.1} & \best{32001.3}{0.1} \\
   & Lamarckian & \best{2910.4}{0.5} & \best{23862.1}{5.8} & \best{412.7}{0.5} & \boldscore{1247.8}{2.6} & \best{2822.2}{0.0} & \best{32001.3}{0.1} \\
   & LB & \best{2910.4}{0.5} & \boldscore{23861.6}{6.1} & \best{412.7}{0.5} & \boldscore{1247.6}{2.7} & \boldscore{2822.1}{0.1} & \best{32001.3}{0.1} \\
\bottomrule
\end{tabular}}
\end{table}

\textbf{Solution Quality.} Following the grid search protocol described in Section~\ref{sec:protocol}, we identified the optimal EA parameters (provided in Appendix~\ref{app:experimental_details} Table~\ref{tab:evolution_params_combined}) and report the corresponding mean scores for MIS and MC in Table~\ref{tab:performance}.

Across all datasets and both problems, every EA evolution type substantially outperforms the four deep learning baselines (GIN, GT, MLP, DeepSets). The following discussion therefore focuses on the comparison with the solvers.

In Table~\ref{tab:performance}, we observe that in the great majority of cases, the best-performing $\probLB \in \{0.15, 0.3, 0.5, 0.9\}$ for LB tracks whichever evolution type dominates: $0.15$ on MIS and $0.9$ on MC, pulling LB towards the stronger single evolution type (Appendix~\ref{app:experimental_details} Table~\ref{tab:evolution_params_combined}). Hence, in our experiment, LB offers no systematic advantage: its best configuration simply converges to the highest performing evolution, and can be discarded in favour of the simpler single-mode alternatives.

Darwinian evolution is the weakest EA across all datasets, with both lower mean scores and higher standard deviations. On MIS, the gap to the best repair-based EA ranges from $2\%$ on BA Small to $27\%$ on ER Large, suggesting that without local search, evolutionary operators alone are less effective at navigating the highly constrained MIS feasibility landscape. On MC the gap is smaller (within $5\%$ of the best EA on all datasets).

For MIS, Baldwinian evolution achieves the strongest performance, within standard deviation, on all datasets. It also nearly matches the specialised KaMIS solver: reaching $98.21\%$ of the KaMIS score on ER Small, $93.18\%$ on ER Large, $99.98\%$ on BA Small, $99.92\%$ on BA Large. For MC, Lamarckian and Baldwinian methods reach similar scores with differences within standard deviation. Taken together, Baldwinian evolution is the most consistent choice: it achieves near-optimal MIS performance, matches the best evolution type on MC, and does so with the lowest seed sensitivity of all EA types.

Surprisingly for MC, all four EAs outperform the solver on ER Small. Also, on ER Large, the best EA reaches $99.9\%$ of the solver's score and $99.0\%$ on BA Large, showing that local-search-augmented EAs can be competitive in practice while operating at a fraction of the computational cost (see Table~\ref{tab:runtime} in Appendix~\ref{app:runtime}). 

Beyond the results shared in Table~\ref{tab:performance}, we recorded nine distinct graphs in ER-Large for which EAs found independent sets that are larger than those detected by KaMIS: six using Baldwinian evolution and three using LB evolution ($\probLB = 0.15$). This is noteworthy given that KaMIS is a solver specifically engineered for MIS, combining kernelisation, local search, and branch-and-reduce, whereas our EAs are general-purpose methods with a simple repair mechanism and no problem-specific design. The fact that such a generalist approach can match and occasionally exceed a dedicated specialist solver exemplifies the effectiveness of generalist evolutionary search for combinatorial optimisation.

\textbf{Convergence.} Figure~\ref{fig:convergence} in Appendix~\ref{app:convergence} shows the mean fitness against the number of fitness evaluations. On MIS, Baldwinian evolution converges faster than LB, with both reaching comparable optima (Table~\ref{tab:performance}, with RB Large as the sole exception). On MC, all three local search EAs reach near-identical final scores and their convergence curves are largely indistinguishable, consistent with the near-tied values in Table~\ref{tab:performance}; Baldwinian evolution converges marginally slower on some instances. Across both problems, Darwinian evolution requires more fitness evaluations to converge and consistently plateaus at a lower optimum. Overall, local search mechanisms consistently accelerate convergence, with Baldwinian evolution offering the best speed-quality trade-off.

\textbf{Solution and Population Diversity.} We report the diversity of solutions produced by each evolution type in Table~\ref{tab:unique_solutions} in Appendix~\ref{app:diversity_optim_sol}, where we provide the mean number of unique optimal solutions found during a run. Note that for Baldwinian evolution, unique solutions are tracked from the locally-searched offspring before they are discarded, ensuring a fair comparison across evolution modes. Darwinian evolution consistently finds very few unique solutions, confirming that local search not only improves mean score but also broadens solution diversity. Figure~\ref{fig:diversity} in Appendix~\ref{app:population_diversity} shows population diversity, measured as mean pairwise Hamming distance, over fitness evaluations. Baldwinian evolution consistently maintains higher population diversity than Lamarckian evolution throughout the runs. For MIS, Baldwinian evolution sustains high population diversity throughout and simultaneously achieves the highest mean scores and the highest solution diversity (Tables~\ref{tab:performance} and ~\ref{tab:unique_solutions}). However, diversity alone does not guarantee performance: Darwinian evolution maintains higher diversity than Lamarckian yet reaches lower optima, and LB sustains high diversity on ER Large without achieving the highest mean score. For MC, all local-search-based variants achieve similar performance (Table~\ref{tab:performance}), yet Lamarckian and LB find substantially more unique solutions (Table~\ref{tab:unique_solutions}) despite very low diversity.

\textbf{Computational Cost.} Table~\ref{tab:runtime} in Appendix~\ref{app:runtime} reports the mean execution time (in seconds). These results have been obtained by running on a sample of $10$ graphs, $3$ seeds. Since all EAs share the same $40{,}000$ fitness evaluation budget, runtime differences reflect the per-evaluation overhead of each strategy. As expected, Darwinian evolution is the fastest, benefiting from the absence of local search overhead. Conversely, Baldwinian evolution is the most computationally expensive, while keeping to a computational cost that is realistic in practice. Lamarckian evolution significantly mitigates the local search cost by keeping locally optimised individuals in subsequent generations.

Yet across all EA variants, absolute runtimes remain negligible compared to the MC Gurobi solver, which runs with a $3600$-second timeout per graph. Even Baldwinian evolution, the most expensive EA type, finishes in under $18$ seconds on RB Large, yielding speedups of $200\times$-$2{,}000\times$ over Gurobi across datasets. 

\begin{table}[t]
\centering

\caption{Parameter universality cost (\%, $\downarrow$): test-set relative loss from a single configuration across datasets and problems instead of per (dataset, problem) tuning. Bold typesets the most robust EA. Cross-problem configurations ($\mu, \lambda, \rc$): Darwinian ($50, 250, 1.0$); Baldwinian ($250, 50, 0.9$); Lamarckian ($10, 10, 0.0$); LB ($250, 250, 1.0$) ($p_{\text{LB}}=0.15$).}
\label{tab:normalized_agg_loss}
\resizebox{\textwidth}{!}{%
\setlength{\tabcolsep}{5pt}%
\begin{tabular}{llrrrrrrr}
\toprule
\textbf{Problem} & \textbf{Evol. Type} & \textbf{ER Small} & \textbf{ER Large} & \textbf{BA Small} & \textbf{BA Large} & \textbf{RB Small} & \textbf{RB Large} & \textbf{Avg.} \\
\midrule
   \multirow{4}{*}{MIS} & Darwinian & $4.98$ & $0.00$ & $2.74$ & $0.00$ & $2.59$ & $\mathbf{-0.11}$ & $1.70$ \\
    & Baldwinian & $0.22$ & $0.00$ & $0.00$ & $0.00$ & $\mathbf{0.00}$ & $0.09$ & $0.05$ \\
    & Lamarckian & $0.47$ & $\mathbf{-0.18}$ & $0.00$ & $\mathbf{-0.06}$ & $\mathbf{0.00}$ & $0.00$ & $0.04$ \\
    & LB & $\mathbf{0.00}$ & $-0.13$ & $\mathbf{-0.01}$ & $0.00$ & $0.10$ & $0.18$ & $\mathbf{0.02}$ \\
\midrule
   \multirow{4}{*}{MC} & Darwinian & $0.29$ & $0.53$ & $0.84$ & $1.65$ & $0.10$ & $0.89$ & $0.72$ \\
    & Baldwinian & $\mathbf{0.00}$ & $\mathbf{0.01}$ & $\mathbf{0.03}$ & $\mathbf{0.11}$ & $\mathbf{0.00}$ & $0.00$ & $\mathbf{0.02}$ \\
    & Lamarckian & $0.81$ & $0.68$ & $1.34$ & $1.85$ & $0.05$ & $0.04$ & $0.80$ \\
    & LB & $0.02$ & $0.05$ & $0.10$ & $0.29$ & $0.00$ & $\mathbf{0.00}$ & $0.08$ \\
\bottomrule
\end{tabular} }
\end{table}

\textbf{Cross-Problem Parameter Robustness.} Running a separate parameter search per problem and dataset is costly. We therefore quantify the loss incurred by using a single configuration per evolution type across problems. For each evolution type, we identify the cross-problem best configuration and evaluate it on the held-out test set against the per-(dataset, problem) best configuration; more details in Appendix~\ref{app:universality}. The relative loss incurred by using the cross-problem configuration is reported in Table~\ref{tab:normalized_agg_loss}. The overall losses are small, confirming that cross-problem parameter transfer is viable for all evolution types, and the negative values indicate that the cross-problem configuration occasionally generalises better than per-pair tuning on unseen graphs. Baldwinian evolution is the most robust, with an average loss of $0.05\%$ on MIS and $0.02\%$ on MC and a maximum of $0.22\%$. LB evolution is comparable, achieving the lowest average loss on MIS ($0.02\%$). Lamarckian evolution incurs a low average loss on MIS ($0.04\%$) but reaches $0.80\%$ on MC, mostly from the BA datasets. Darwinian evolution is the most sensitive, reaching an average loss of $1.70\%$ on MIS and $0.72\%$ on MC.

\section{Theoretical Analysis}

We complement our experimental findings with a mathematical runtime analysis on a natural generalisation of the established \deceptiveleadingblocks benchmark introduced by Lehre and Nguyen~\cite{LehreN19foga}. Our results will align with our experimental findings, but more importantly, with the clarity of the mathematical analysis, they will also explain why Baldwinian evolution can be superior to Lamarckian and Darwinian evolution. 

\textbf{The \deceptiveleadingblocks Benchmark.}
We regard the following benchmark problem, which extends the \deceptiveleadingblocks benchmark proposed by Lehre and Nguyen~\cite{LehreN19foga} in a natural fashion, namely from block length $k=2$ to arbitrary block lengths~$k \ge 2$. Let $k \in \N_{\ge 2}$ and $n$ be a multiple of $k$. We view the set $[1..n],$ denoting the integers one to $n,$ of bit-positions as partitioned, in a left-to-right fashion, into blocks of length~$k$. The objective value (fitness) of a bit-string $x$ is strongly influenced by the number of blocks, contiguously counted from left to right, that contain only ones. The subsequent block still contributes to the fitness, but to a lesser degree and in a deceptive manner, that is, the contribution is larger when the block contains fewer ones. Consequently, bit-strings composed of a sequence of all-ones blocks followed by an all-zero block are true local optima.

More formally, for a fixed value of $k$ mostly suppressed in the following notation, and for all $\ell \in [1..n/k]$ we call $B_\ell = [(\ell-1)k+1..\ell k]$ the \emph{$\ell$-th block}. We write $x_{B_\ell} = (x_{(\ell-1)k+1}, \dots, x_{\ell k}) \in \{0,1\}^k$ to denote the \emph{$\ell$-th block} of $x \in \{0,1\}^n$. For $x \neq (1, \dots, 1)$, we let $c(x) := \min\{\ell \in [1..n/k] \mid x_{B_\ell} \neq (1, \dots, 1)\}$ denote the first block that is not all-ones and call $x_{B_{c(x)}}$ the \emph{critical block of $x$}. For $x = (1, \dots, 1)$, we artificially define $c(x)=n/k+1$. With these definitions, the \deceptiveleadingblocks function with block length~$k$ is defined by
\[
\DLB(x) := \DLB_k(x) := k (c(x)-1) + (k - 1 - \|x_{B_{c(x)}}\|_1)
\]
for all $x \in \{0,1\}^n$. We see that the $c(x)-1$ leading all-ones blocks each contribute $k$ to the fitness, whereas the critical block in a deceptive manner contributes $k - 1 -\|x_{B_{c(x)}}\|_1 \in [0..k-1]$.

As said above, the \deceptiveleadingblocks was introduced in~\cite{LehreN19foga} for block length $k=2$ and has received a considerable amount of attention by the theory community since then. Probably, it is the third most prominent multi-modal benchmark in the runtime analysis community, behind the \jump and \cliff benchmarks. It differs from these in that a typical EA runs into local optima more frequently than just once.

Lehre and Nguyen~\cite{LehreN19foga} used the \DLB benchmark to demonstrate that estimation-of-distribution algorithms (EDAs) can have enormous difficulties with local optima, shortly after \cite{Doerr19gecco,Doerr21cgajump} had shown that they cope surprisingly well with the local optimum of the \jump and \cliff benchmark. Specifically, \cite{LehreN19foga} proved that the UMDA algorithm with parameter settings $\mu = \Theta(\lambda)$ and $\lambda = o(n)$ has an expected runtime of at least $e^{\Omega(\lambda)}$ on DLB, in contrast to the $O(n^3)$ runtime guarantees they showed for many standard evolutionary algorithms. This clear message was  questioned in~\cite{DoerrK21ecj}, where it was shown that the UMDA with parameters avoiding genetic drift~\cite{DoerrZ20tec} can optimise DLB in $O(n^2 \log n)$ function evaluations with high probability. Subsequently, a $\Theta(n^2)$ expected runtime was shown for the Metropolis algorithm~\cite{WangZD24} and an $O(n \log n)$ runtime with high probability for the significance-based compact genetic algorithm from~\cite{DoerrK20tec}.

\textbf{Mathematical Runtime Analysis.}
Our main results are the following, asymptotically tight runtime guarantees.
\begin{theorem}\label{thm:main}
Let $k \in \N_{\ge 2}$. Then the expected runtimes, measured in iterations, on the \deceptiveleadingblocks benchmark, of the \oea with Darwinian, Lamarckian, and Baldwinian evolution satisfy the following estimates.
\begin{itemize}
\item Darwin: $\Theta(n^{k+1})$.
\item Lamarck: $\Theta(n^{k})$.
\item Baldwin: $\Theta(n^2)$.
\end{itemize}
Here the asymptotics are for growing problem size $n$ and constant $k$.
\end{theorem}

Note critically that we stated the runtimes in terms of iterations as this is the most complete information. When comparing the runtimes over different evolution types, it is important to take into account that an iteration of the Darwinian \oea mostly consists of a single fitness evaluation (of the offspring), whereas one iteration in the two other evolution types consists mostly of running the local search procedure once on the result of mutation. We discuss the implications of this after the sketch proof of Theorem~\ref{thm:main}.

\textbf{Proof Sketch:} To present the proof ideas, let us make precise the algorithm we regard. The \oea works with a parent population of size one (which is why we call this individual the \emph{parent individual}). It is initialised randomly. Then, in each iteration of the main loop, an offspring is generated from the parent via bit-wise mutation with mutation rate~$1/n$, that is, by flipping each bit independently with probability~$1/n$. If the offspring is at least as good as the parent (in terms of the fitness), then it is \emph{accepted}, that is, replaces the parent; otherwise it is discarded. As common in the mathematical runtime analysis of evolutionary algorithms~\cite{NeumannW10,AugerD11,Jansen13,ZhouYQ19,DoerrN20}, we do not specify a termination criterion as we are interested in how long it takes until an optimal solution is generated (via mutation or a local search routine). The pseudo\-code of this algorithm can be found in  Algorithm~\ref{alg:oea}. Precisely, this is the Darwinian version of the \oea. For the Lamarckian one, any solution that is generated, that is, both the random initial individual and all offspring, is immediately replaced by the result from optimising it via a local search routine. 
For the Baldwinian \oea, the individuals are left unchanged, but as their fitness we use the fitness of the result of applying the local search routine to the individual. In Theorem~\ref{thm:main}, we assume that the local search routine is the best-improvement hillclimber with respect to the standard neighbourhood structure of the discrete hypercube, that is, we repeat checking all neighbours (solutions differing in exactly one bit) of our current solution and replacing it with the best one (breaking ties randomly) until no neighbour is strictly better.  
\begin{algorithm2e}[t]
	Let $x$ be chosen uniformly at random from $\{0,1\}^n$\;
	\Repeat{forever}{%
		$y \assign \textsc{Mutate}(x)$\;
		\lIf{$f(y) \geq f(x)$}{$x \assign y$}
	}
	\caption{The \oea to maximise $f : \{0,1\}^n \to \R$.}
	\label{alg:oea}
\end{algorithm2e}

The proof of our result is slightly too technical to give a short summary here, so let us only describe some arguments that show the differences between the three algorithm variants. The main difference is how the three algorithms can leave the local optima of the $\DLB_k$ problem. Here the Darwinian EA has no other chance than to flip the $k$ zeros of the critical block to ones as all nearer solutions have a lower fitness. Flipping $k$ particular bits takes time $\Theta(n^k)$. The remaining factor of $\Theta(n)$ stems from the fact that $\Theta(n)$ times a local optimum has to be left. For the Lamarckian EA, when trapped in a local optimum, flipping $k-1$ of the $k$ zeros in the critical block suffices. The resulting solution is a neighbour of a solution having one more leading all-ones block, hence the local search routine will move to that solution and possibly to even better solutions, but in any case with this mutation the local optimum is left. Flipping $k-1$ out of $k$ fixed bits takes time $\Theta(n^{k-1})$, which explains the factor-$n$ speed-up over the Darwinian EA. In contrast, the Baldwinian EA does not see any local optima. Imagine that its current solution $x$ is a local optimum with $j$ leading all-ones blocks. This clearly has fitness $jk + k-1$ also for the Baldwinian EA since local search does not find a better solution. However, any solution differing from $x$ in up to $k-2$ bits of the critical block has the same Baldwinian fitness since local search would return to $x$  (only for the determination of the fitness!). This allows the Baldwinian EA to perform an unbiased walk on this plateau of constant Baldwinian fitness. When this walk has reached a solution with $k-2$ ones in the critical block, then an easy one-bit flip gives a solution in the basin of attraction of a strictly better solution. This constitutes an increase of the Baldwinian fitness, and prevents the algorithm to return to the previous local optimum. That this way of leaving the local optimum only takes $O(n)$ time is not totally easy to prove and we use an existing runtime analysis for generalised Needle functions~\cite{DoerrK25} for this purpose. We note that we over-simplified the situation a little here -- the true plateau of constant Baldwinian fitness is much larger since also solutions with incomplete earlier blocks can be accepted if these blocks contain $k-1$ ones. We have to leave the details to the following formal proof, which can be found in Appendix \ref{app:proofs}.

\textbf{Towards a Fair Comparison.} We now discuss what the result in Theorem~\ref{thm:main} means when taking into account the different costs of one iteration due to the local search procedure. Such a comparison is non-trivial since it depends on how efficiently the local search procedure is implemented. As seen in the experimental part of this work, with a moderate understanding of the underlying problem, local search can be implemented significantly more efficiently than via $n$ fitness evaluations per local search step. Since this depends on the particular problem, and on the extent to which it is understood, we here take the pessimistic (for Lamarckian and Baldwinian evolution) view that one iteration of the local search procedure uses $n$ fitness evaluations, and we measure the performance in terms of fitness evaluations.

In the Lamarckian \oea, the parent is a local optimum and the offspring differs from it in an expected constant number of bits. This implies that an expected constant number of previous all-ones blocks were destroyed, and each need one iteration of local search to be repaired (unless the repair gets stuck prematurely). If the current critical block is not modified so that local search transforms it into an all-ones block, then at most $k-1 = O(1)$ local search iterations may touch this block. Only with probability $O(n^{-(k-1)})$, the current critical block is modified so that local search improves it (in one iteration) to an all-ones block. Above that, only $O(n)$ fitness levels remain, so at most $O(n)$ further iterations of the local search procedure are possible (in reality, this number, in expectation, is only $O(1)$, but we do not need this here). In summary, we see that each local search call in the Lamarckian \oea performs an expected number of $O(1)$ iterations. When counting one iteration of the local search procedure as $n$ fitness evaluations, the runtime of the Lamarckian \oea would be $O(n^{k+1})$, on par with the Darwinian \oea. 

For the Baldwinian \oea, estimating the cost of the local search procedure is a little harder. As an upper bound, we can argue that we have only $n+1$ fitness levels, hence at most $n$ iterations per call of the local search procedure, hence per iteration of the \oea. Unfortunately, this estimate is asymptotically tight. Consider a block and a time where this block contributes to the Baldwinian fitness, that is, local search had transformed this block and all previous ones into all-ones blocks. Since we have an elitist algorithm, this property remains fulfilled throughout the further run of the algorithm. In particular, this block will always contain either $k-1$ or $k$ ones. Since the algorithm has no strong preference between these two states, in the long run, both states will occur with constant rate (to make this argument precise, we would need to compute the transition probabilities, which are both $\Theta(1/n)$ for switching the state, and the mixing time, which is $O(n)$ as follows from elementary properties of two-state Markov chains). Hence the expected number of blocks that contribute to the fitness, but have $k-1$ ones (and thus need a local search iteration), of a solution with fitness $\Theta(n)$ is $\Theta(n)$. Hence in a constant fraction of the iterations, the local search procedure performs $\Omega(n)$ iterations. When counting a local search iteration as $n$ fitness evaluations, this leads to a runtime of $\Theta(n^4)$ fitness evaluations for the Baldwinian \oea. Hence in this view, the Baldwinian \oea is worse than the two others for the easiest case of $k=2$, all three are asymptotically equal for $k=3$, and Baldwinian evolution outperforms the others for all $k \ge 4$.

In a practical application, problem-specific information might allow for more efficient implementations of the local search routine, either because not all neighbours of the current solution need to be inspected, or because the fitness of a solution known to be a neighbour can be computed more efficiently than a general fitness evaluation. In this case, the Lamarckian \oea would have an asymptotic advantage over the Darwinian one for all $k$. Since an iteration in any case takes at least one fitness evaluation, the runtime would still be at least $\Omega(n^k)$ fitness evaluations, showing that the advantage of the Lamarckian \oea is limited to a factor of $O(n)$, and that the Baldwinian \oea also without accounting for such improvements outperforms it for $k \ge 5$. 

Overall, a precise performance comparison depends critically on the local search procedure's cost. The Baldwinian \oea will outperform the others from a certain small value of $k$ on. This aligns with the general belief in the community that Baldwinian evolution is more suitable for more difficult problems.

\section{Conclusion}

We have taken a fresh look at Darwinian, Baldwinian, Lamarckian, and LB evolution, which have mostly existed independently in the EA literature without empirical or formal theoretical comparisons. On the empirical side, we demonstrated that local search augmentation consistently improves performance over Darwinian evolution. Baldwinian evolution is the most consistent choice overall, while Lamarckian evolution is a competitive alternative on MC. The LB hybrid offers no systematic advantage over either. Finally, we identified a generalist parameter configuration per evolution type, providing a practical off-the-shelf starting point for future studies. On the theoretical side, we extended the \DLB benchmark to arbitrary block length $k \ge 2$ and proved tight asymptotic runtime bounds for the \oea: Baldwinian evolution runs in $\Theta(n^2)$ iterations, Lamarckian evolution in $\Theta(n^k)$, and Darwinian evolution in $\Theta(n^{k+1})$.

When accounting for the cost of the local search procedure in fitness evaluations, the runtime ranking depends on implementation, but Baldwinian evolution outperforms the others from a small value of $k$ onwards. These results have clear practical implications: local search augmentation should more frequently be considered over Darwinian evolution; in particular, Baldwinian evolution is the most promising starting point in the absence of problem-specific knowledge. In addition, our generalist parameter configurations for the different evolution types could be a useful practical reference.

\subsection*{Acknowledgements}
This work was supported by FMJH Program PGMO and the French National Research Agency (ANR) via the ``GraspGNNs'' JCJC grant (ANR-24-CE23-3888).

\bibliographystyle{splncs04}
\bibliography{bib/alles_ea_master,bib/ich_master,bib/citations}

\newpage
\appendix

\section{Encoding Requirements for Lamarckian Evolution}
\label{app:lamarck}

We begin by supplementing our review of related work with a note on a subtle advantage that Baldwinian evolution has been observed to gain. A critical requirement for Lamarckian evolution is that improvements found by local search must be successfully written back into the encoding of an individual. In the combinatorial settings considered in this work, this mapping is trivial since local search operates directly on the encoding, and the improved solution can therefore replace the offspring. In other domains, such as evolutionary robotics, the encoding describes how to build an individual rather than the individual itself, and local search operates on the resulting structure \cite{LuoMTE23,luo2025}. Translating improvements back into the encoding is then a non-trivial design problem, and is sometimes impossible. Baldwinian evolution leaves the encoding untouched and only modifies the fitness of the individual, which is a structural advantage in settings with complex representations.

\section{Problem Formulations}
\label{app:problems}

In this section, we formally define the two combinatorial optimisation problems studied in this paper. Recall that we let $G = (V, E)$ be an undirected graph, where $V = \{v_1, \dots, v_n\}$ is the set of vertices and $E \subseteq \{ \{v_i, v_j\} : v_i, v_j \in V\}$ is the set of edges. A candidate solution is represented by a binary vector $x \in \{0, 1\}^n$, where $x_i = 1$ if vertex $v_i$ is selected and $x_i = 0$ otherwise.

\subsection{Maximum Independent Set (MIS)}
An independent set of a graph $G$ is a subset of vertices $S \subseteq V$ such that no two vertices in $S$ are connected by an edge. The MIS problem seeks to find an independent set $S$ of maximal cardinality $|S|$. Using the binary representation $x$, where $x_i = 1$ if $v_i \in S$ and $x_i = 0$ otherwise, the problem can be formulated as
\begin{equation}
    \begin{aligned}
        & \text{maximise} && \phiMIS(x) = \sum_{i=1}^n x_i \\
        & \text{subject to} && x_i + x_j \le 1, \quad \forall \{v_i, v_j\} \in E \\
        & && x_i \in \{0, 1\}, \quad i = 1, \dots, n.
    \end{aligned}
\end{equation}

\subsection{Maximum Cut (MC)}
The Maximum Cut problem involves partitioning the vertex set $V$ into two disjoint sets $S$ and $V \setminus S$. The goal is to maximise the number of edges that have one endpoint in $S$ and the other in $V \setminus S$. Using the binary encoding $x$, where $x_i = 1$ if $v_i \in S$ and $x_i = 0$ if $v_i \in V \setminus S$, the problem can be formulated as
\begin{equation}
    \begin{aligned}
        & \text{maximise} && \phiMC(x) = \sum_{\{v_i, v_j\} \in E} (x_i(1 - x_j) + x_j(1 - x_i)) \\
        & \text{subject to} && x_i \in \{0, 1\}, \quad i = 1, \dots, n.
    \end{aligned}
\end{equation}
Unlike MIS, MC is an unconstrained optimisation problem, as every binary vector $x$ corresponds to a valid partition of the graph.

\section{Additional Experimental Details}
\label{app:experimental_details}

This appendix section collects the technical details that support the empirical analysis in Section~\ref{sec:results}. It includes a note on solver performance (Appendix~\ref{app:benchmark}), implementation choices for the large-scale parameter search (Appendix~\ref{app:implementation}), the full parameter search space and best per-dataset configurations (Tables~\ref{tab:parameter_grid} and~\ref{tab:evolution_params_combined}, Appendix~\ref{app:hyperparams}), convergence trajectories (Figure~\ref{fig:convergence}, Appendix~\ref{app:convergence}), population diversity over runs (Figure~\ref{fig:diversity}, Appendix~\ref{app:population_diversity}), diversity of optimal solutions found (Table~\ref{tab:unique_solutions}, Appendix~\ref{app:diversity_optim_sol}), mean execution times (Table~\ref{tab:runtime}, Appendix~\ref{app:runtime}), and the procedure for selecting the universal parameter combination and computing the universality cost (Appendix~\ref{app:universality}).

\subsection{A Note On the Displayed Solver Performance}
\label{app:benchmark}

While all baseline results come from the GraphBench paper, we recomputed the MIS solver results from the true labels available in the test set provided with the GraphBench benchmark. This recomputation of the results led to a small and ultimately inconsequential, difference in the reported solver values.

\subsection{Implementation Details}
\label{app:implementation}

Given the scale of the GraphBench datasets and the large number of EA runs required by the parameter search ($105{,}840$ runs per problem), and the use of local search steps, computational efficiency is critical. The graph's topology is represented using sparse adjacency matrices in Compressed Sparse Row (CSR) format. Fitness evaluation for both problems is implemented as sparse matrix-vector multiplication via SciPy's sparse routines, reducing the complexity of fitness evaluation from $O(n^2)$ with a dense adjacency matrix to $O(|E|)$ with the sparse format, since only the $|E|$ non-zero entries are visited during the matrix-vector product. This allows for efficient scaling to the larger benchmark instances investigated in this study. The local search and local search logic for both problems is implemented in JIT-compiled Python using Numba to achieve near-native execution speeds. The complete source code reproducing our experiments can be accessed on github: \url{https://github.com/Hypatia-II/polerina}.

\subsection{Parameter Configurations}
\label{app:hyperparams}

Table~\ref{tab:parameter_grid} reports the full parameter search space shared by all four evolution types. 
\begin{table}[h!]
\centering
\caption{Parameter Search Space for Evolutionary Types: Darwinian, Baldwinian, Lamarckian and LB. $^\dagger$Only applied in the LB evolution type.}
\label{tab:parameter_grid}
\small
\begin{tabular}{lc}
\toprule
\textbf{Parameter} & {\textbf{Values Tested}} \\
\midrule
\quad Population Size ($\mu$)                        & \{10, 50, 250\} \\
\quad Initialisation Type                            & Random \\
\quad Number of Offspring ($\lambda$)                & \{10, 50, 250\} \\
\quad Crossover Rate ($\rc$)                         & \{0.0, 0.5, 0.9, 1.0\} \\
\quad Mutation Type                                  & Bernoulli \\
\midrule
\quad Lamarckian Probability ($\probLB$)$^\dagger$   & \{0.15, 0.3, 0.5, 0.9\} \\
\bottomrule
\end{tabular}
\end{table}
\FloatBarrier

Table~\ref{tab:evolution_params_combined} reports the best configuration selected for each evolution type, dataset, and problem after grid search on the training split.
\begin{table}[h!]
\centering
\caption{Best evolution type parameter configuration for MIS and MC. All configurations used Bernoulli mutation and random initialisation. Values shown as $\mu/\lambda/\rc$; Darwinian, Baldwinian, and Lamarckian evolution always have $p_\text{LB}=0$ (---).}
\label{tab:evolution_params_combined}
\footnotesize
\renewcommand{\arraystretch}{1.0}
\begin{tabular}{llcccc}
\toprule
& & \multicolumn{2}{c}{\textbf{MIS}} & \multicolumn{2}{c}{\textbf{MC}} \\
\cmidrule(lr){3-4}\cmidrule(lr){5-6}
\textbf{Dataset} & \textbf{Type} & $\boldsymbol{\mu/\lambda/\rc}$ & $\boldsymbol{p_\text{LB}}$ & $\boldsymbol{\mu/\lambda/\rc}$ & $\boldsymbol{p_\text{LB}}$ \\ \midrule
    \multirow{4}{*}{ER Small} & Darwinian & $250/50/1.0$ & --- & $10/10/0.0$ & --- \\
     & Baldwinian & $250/250/1.0$ & --- & $250/50/1.0$ & --- \\
     & Lamarckian & $250/250/0.0$ & --- & $250/250/1.0$ & --- \\
     & LB & $250/250/1.0$ & 0.15 & $250/250/1.0$ & 0.9 \\
\cmidrule{1-6}
    \multirow{4}{*}{ER Large} & Darwinian & $50/250/1.0$ & --- & $10/10/0.5$ & --- \\
     & Baldwinian & $50/50/1.0$ & --- & $250/50/1.0$ & --- \\
     & Lamarckian & $250/250/0.0$ & --- & $250/250/1.0$ & --- \\
     & LB & $250/10/1.0$ & 0.15 & $250/250/1.0$ & 0.9 \\
\cmidrule{1-6}
    \multirow{4}{*}{BA Small} & Darwinian & $250/250/1.0$ & --- & $10/10/0.5$ & --- \\
     & Baldwinian & $250/50/0.9$ & --- & $250/250/0.9$ & --- \\
     & Lamarckian & $10/10/0.0$ & --- & $250/250/0.9$ & --- \\
     & LB & $250/250/0.9$ & 0.15 & $250/250/0.9$ & 0.9 \\
\cmidrule{1-6}
    \multirow{4}{*}{BA Large} & Darwinian & $50/250/1.0$ & --- & $10/10/1.0$ & --- \\
     & Baldwinian & $250/10/1.0$ & --- & $250/250/1.0$ & --- \\
     & Lamarckian & $250/10/1.0$ & --- & $250/250/0.9$ & --- \\
     & LB & $250/250/1.0$ & 0.15 & $250/250/0.9$ & 0.9 \\
\cmidrule{1-6}
    \multirow{4}{*}{RB Small} & Darwinian & $250/10/0.5$ & --- & $250/250/1.0$ & --- \\
     & Baldwinian & $250/50/0.9$ & --- & $10/10/1.0$ & --- \\
     & Lamarckian & $10/10/0.0$ & --- & $50/250/1.0$ & --- \\
     & LB & $250/10/1.0$ & 0.15 & $50/50/0.9$ & 0.5 \\
\cmidrule{1-6}
    \multirow{4}{*}{RB Large} & Darwinian & $50/250/0.9$ & --- & $10/50/1.0$ & --- \\
     & Baldwinian & $250/250/1.0$ & --- & $250/250/0.9$ & --- \\
     & Lamarckian & $10/10/0.0$ & --- & $250/50/0.9$ & --- \\
     & LB & $50/10/1.0$ & 0.15 & $250/10/0.9$ & 0.9 \\
\bottomrule
\end{tabular}
\end{table}
\FloatBarrier

\subsection{Convergence}
\label{app:convergence}
We now provide the mean convergence speed results in Figure \ref{fig:convergence}.
\begin{figure}[h!]
    \centering
    \begin{subfigure}{\textwidth}
        \includegraphics[width=\textwidth]{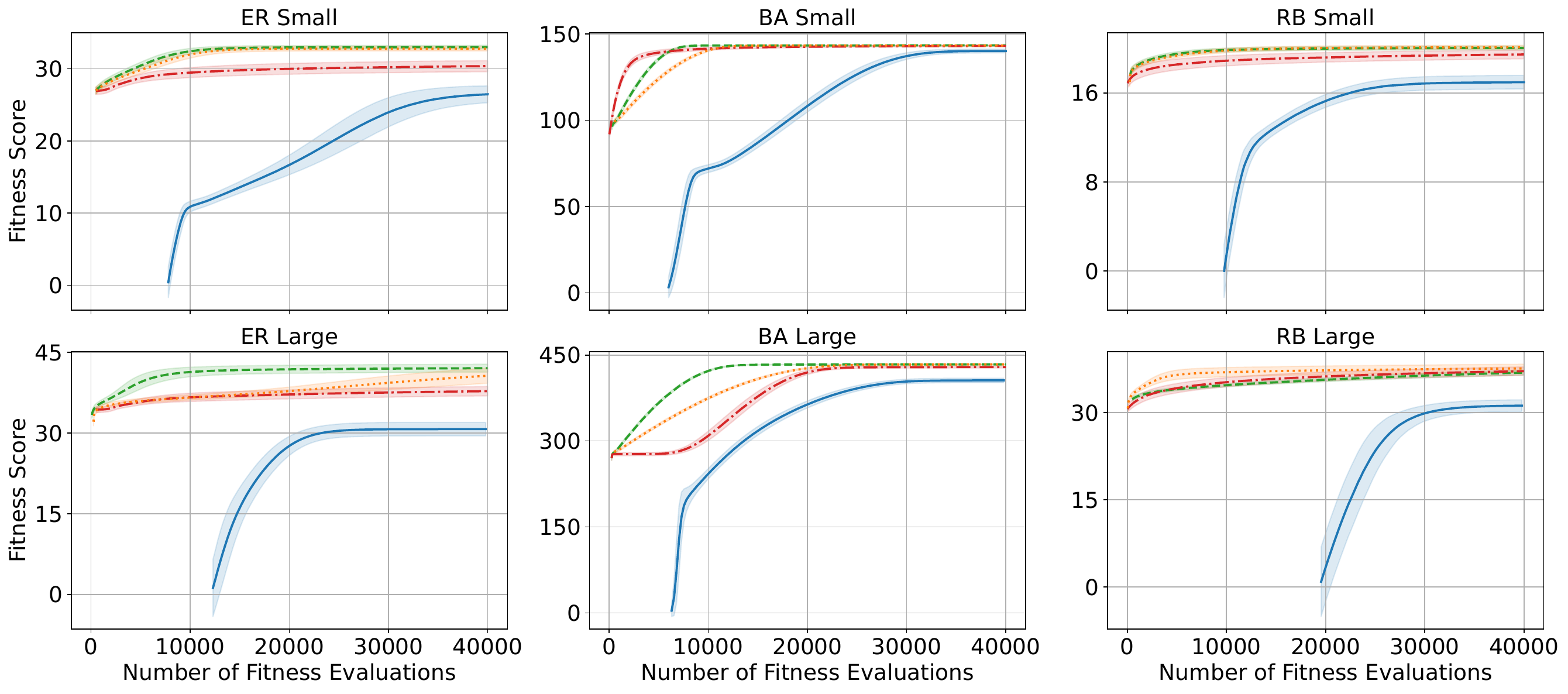}
        \caption{MIS}
        \label{fig:mis_convergence}
    \end{subfigure}
    \begin{subfigure}{\textwidth}
        \includegraphics[width=\textwidth]{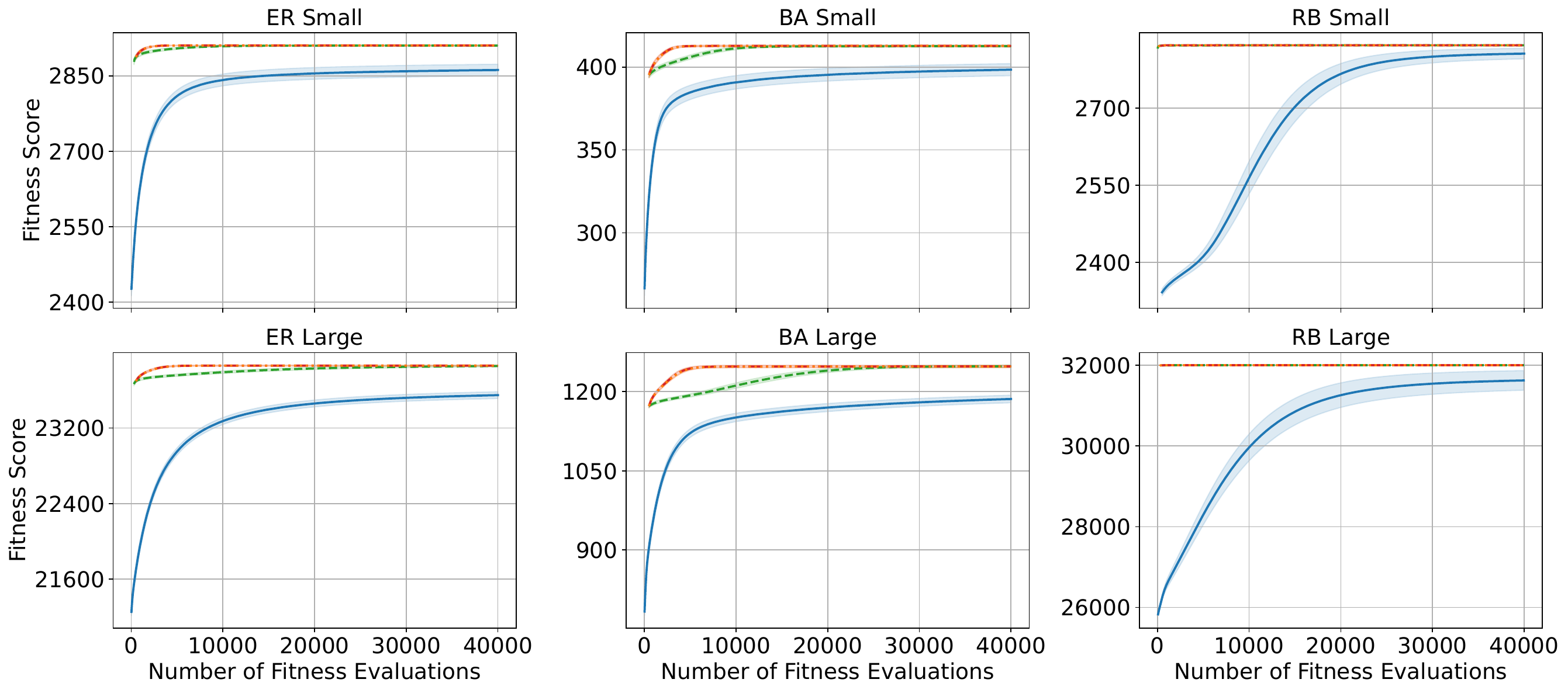}
        \caption{MC}
        \label{fig:mc_convergence}
    \end{subfigure}
    \caption{Mean fitness score ($\uparrow$) over fitness evaluations per EA evolution across datasets for MIS and MC. Curves use the per-(dataset, problem) best parameter configuration. Darwinian (solid blue), Baldwinian (dashed green), Lamarckian (dash-dot red), LB (dotted orange).}
    \label{fig:convergence}
\end{figure}
\FloatBarrier

\newpage
\subsection{Diversity of Optimal Solutions}
\label{app:diversity_optim_sol}

Table~\ref{tab:unique_solutions} provides results on the diversity of the optimal solutions found by the different evolution types.

\begin{table}[h!]
\centering
\caption{Mean unique optimal solutions ($\uparrow$, $\pm$ std) per EA evolution across datasets for MIS and MC. Bold typesets the most diverse EA.}
\label{tab:unique_solutions}
\resizebox{\textwidth}{!}{%
\begin{tabular}{llrrrrrr}
\toprule
\textbf{Problem} & \textbf{Evol. Type} & \textbf{ER Small} & \textbf{ER Large} & \textbf{BA Small} & \textbf{BA Large} & \textbf{RB Small} & \textbf{RB Large} \\
\midrule
   \multirow{4}{*}{MIS} & Darwinian & $3.7$ {\scriptsize $\pm$ $2.4$} & $1.5$ {\scriptsize $\pm$ $0.6$} & $5.7$ {\scriptsize $\pm$ $2.7$} & $1.7$ {\scriptsize $\pm$ $0.7$} & $15.0$ {\scriptsize $\pm$ $10.2$} & $3.2$ {\scriptsize $\pm$ $2.0$} \\
    & Baldwinian & $\mathbf{71.9}$ {\scriptsize $\pm$ $\mathbf{56.6}$} & $\mathbf{94.9}$ {\scriptsize $\pm$ $\mathbf{101.1}$} & $\mathbf{14290.2}$ {\scriptsize $\pm$ $\mathbf{1156.0}$} & $\mathbf{6055.2}$ {\scriptsize $\pm$ $\mathbf{2696.0}$} & $\mathbf{12651.4}$ {\scriptsize $\pm$ $\mathbf{3766.2}$} & $1933.8$ {\scriptsize $\pm$ $1278.8$} \\
    & Lamarckian & $28.6$ {\scriptsize $\pm$ $21.6$} & $24.2$ {\scriptsize $\pm$ $16.6$} & $442.8$ {\scriptsize $\pm$ $207.3$} & $329.8$ {\scriptsize $\pm$ $152.8$} & $1388.3$ {\scriptsize $\pm$ $549.7$} & $856.3$ {\scriptsize $\pm$ $477.7$} \\
    & LB & $57.0$ {\scriptsize $\pm$ $47.9$} & $34.3$ {\scriptsize $\pm$ $41.4$} & $7915.6$ {\scriptsize $\pm$ $886.7$} & $2276.3$ {\scriptsize $\pm$ $1138.7$} & $7651.9$ {\scriptsize $\pm$ $2541.1$} & $\mathbf{5739.1}$ {\scriptsize $\pm$ $\mathbf{2545.3}$} \\
\midrule
   \multirow{4}{*}{MC} & Darwinian & $7.3$ {\scriptsize $\pm$ $5.7$} & $3.7$ {\scriptsize $\pm$ $2.6$} & $1133.0$ {\scriptsize $\pm$ $753.8$} & $405.1$ {\scriptsize $\pm$ $335.7$} & $99.1$ {\scriptsize $\pm$ $66.4$} & $49.1$ {\scriptsize $\pm$ $25.3$} \\
    & Baldwinian & $8.3$ {\scriptsize $\pm$ $4.0$} & $2.5$ {\scriptsize $\pm$ $2.0$} & $14475.7$ {\scriptsize $\pm$ $3012.8$} & $2253.4$ {\scriptsize $\pm$ $1870.7$} & $8781.2$ {\scriptsize $\pm$ $350.2$} & $8974.9$ {\scriptsize $\pm$ $734.5$} \\
    & Lamarckian & $\mathbf{9.0}$ {\scriptsize $\pm$ $\mathbf{3.9}$} & $\mathbf{8.0}$ {\scriptsize $\pm$ $\mathbf{6.4}$} & $\mathbf{30502.3}$ {\scriptsize $\pm$ $\mathbf{1931.0}$} & $\mathbf{27529.0}$ {\scriptsize $\pm$ $\mathbf{2885.1}$} & $10935.8$ {\scriptsize $\pm$ $323.4$} & $\mathbf{12030.7}$ {\scriptsize $\pm$ $\mathbf{461.3}$} \\
    & LB & $\mathbf{9.0}$ {\scriptsize $\pm$ $\mathbf{3.9}$} & $\mathbf{8.0}$ {\scriptsize $\pm$ $\mathbf{6.4}$} & $30412.5$ {\scriptsize $\pm$ $2015.7$} & $27449.5$ {\scriptsize $\pm$ $2967.7$} & $\mathbf{11250.6}$ {\scriptsize $\pm$ $\mathbf{199.3}$} & $12024.1$ {\scriptsize $\pm$ $526.9$} \\
\bottomrule
\end{tabular}}
\end{table}
\FloatBarrier

\subsection{Population Diversity}
\label{app:population_diversity}

Figure \ref{fig:diversity} provides results on the population diversity throughout the runs for each evolution types.

\begin{figure}[h!]
    \centering
    \begin{subfigure}{\textwidth}
        \includegraphics[width=\textwidth]{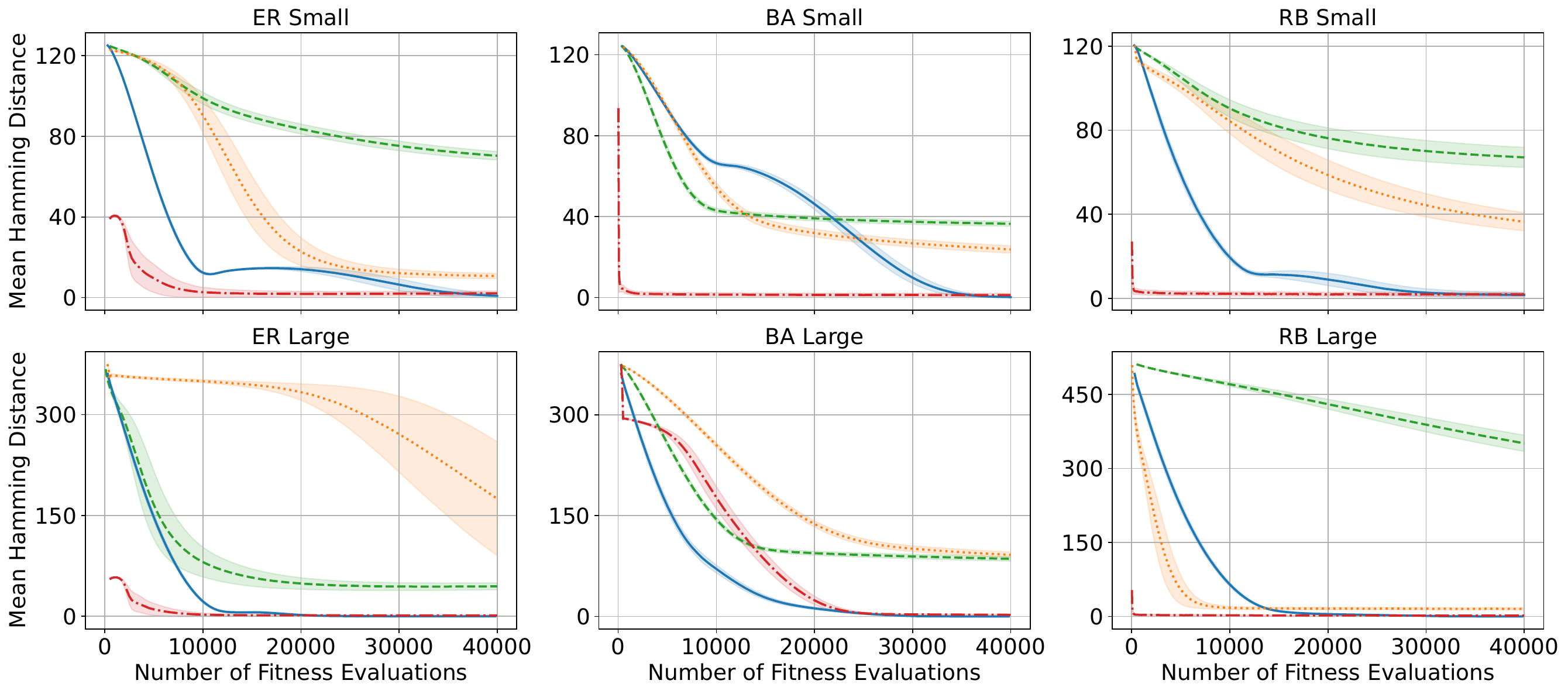}
        \caption{MIS}
        \label{fig:mis_diversity}
    \end{subfigure}
    \begin{subfigure}{\textwidth}
        \includegraphics[width=\textwidth]{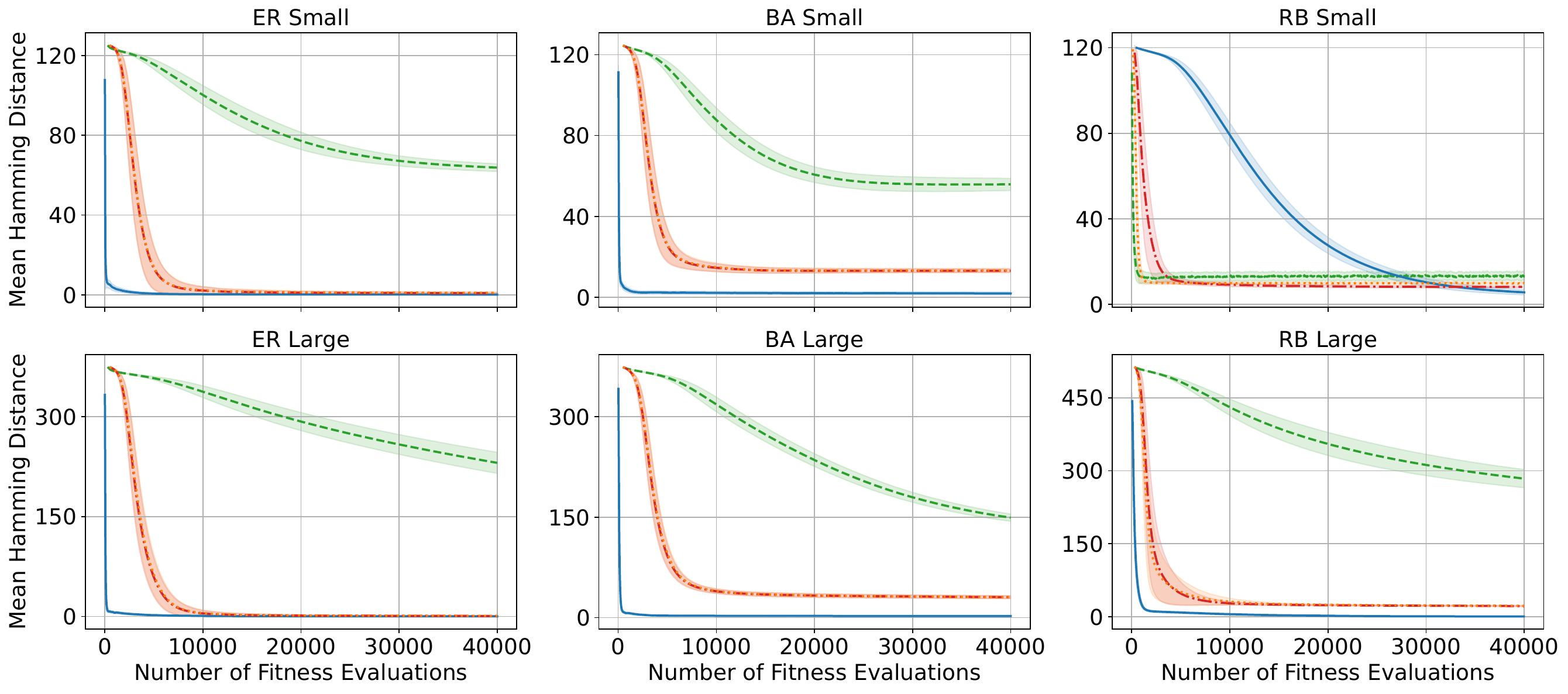}
        \caption{MC}
        \label{fig:mc_diversity}
    \end{subfigure}
    \caption{Mean pairwise Hamming distance as a measure of population diversity over fitness evaluations per EA evolution across datasets for MIS and MC. Curves use the per-(dataset, problem) best parameter configuration. Darwinian (solid blue), Baldwinian (dashed green), Lamarckian (dash-dot red), LB (dotted orange).}
    \label{fig:diversity}
\end{figure}
\FloatBarrier

\subsection{Runtime}
\label{app:runtime}

Table~\ref{tab:runtime} provides runtime results for the different evolution types.

\begin{table}[h!]
\centering
\caption{Mean runtime in seconds ($\downarrow$, $\pm$ std) per EA evolution across datasets for MIS and MC. Bold typesets the fastest EA.}
\label{tab:runtime}
\resizebox{\textwidth}{!}{%
\begin{tabular}{llrrrrrr}
\toprule
\textbf{Problem} & \textbf{Evol. Type} & \textbf{ER Small} & \textbf{ER Large} & \textbf{BA Small} & \textbf{BA Large} & \textbf{RB Small} & \textbf{RB Large} \\
\midrule
   \multirow{4}{*}{MIS} & Darwinian & $\mathbf{0.36}$ {\scriptsize $\pm$ $\mathbf{0.00}$} & $\mathbf{1.97}$ {\scriptsize $\pm$ $\mathbf{0.02}$} & $\mathbf{0.17}$ {\scriptsize $\pm$ $\mathbf{0.00}$} & $\mathbf{0.61}$ {\scriptsize $\pm$ $\mathbf{0.01}$} & $0.75$ {\scriptsize $\pm$ $0.01$} & $\mathbf{2.54}$ {\scriptsize $\pm$ $\mathbf{0.10}$} \\
    & Baldwinian & $1.93$ {\scriptsize $\pm$ $0.45$} & $13.49$ {\scriptsize $\pm$ $0.43$} & $0.64$ {\scriptsize $\pm$ $0.02$} & $4.50$ {\scriptsize $\pm$ $0.13$} & $1.65$ {\scriptsize $\pm$ $0.07$} & $19.69$ {\scriptsize $\pm$ $0.51$} \\
    & Lamarckian & $0.53$ {\scriptsize $\pm$ $0.42$} & $2.12$ {\scriptsize $\pm$ $0.02$} & $0.32$ {\scriptsize $\pm$ $0.00$} & $2.12$ {\scriptsize $\pm$ $0.06$} & $\mathbf{0.56}$ {\scriptsize $\pm$ $\mathbf{0.02}$} & $3.40$ {\scriptsize $\pm$ $0.07$} \\
    & LB & $0.84$ {\scriptsize $\pm$ $0.18$} & $12.17$ {\scriptsize $\pm$ $1.01$} & $0.37$ {\scriptsize $\pm$ $0.02$} & $2.19$ {\scriptsize $\pm$ $0.06$} & $1.29$ {\scriptsize $\pm$ $0.07$} & $4.60$ {\scriptsize $\pm$ $0.20$} \\
\midrule
   \multirow{4}{*}{MC} & Darwinian & $\mathbf{0.44}$ {\scriptsize $\pm$ $\mathbf{0.00}$} & $\mathbf{2.62}$ {\scriptsize $\pm$ $\mathbf{0.01}$} & $0.26$ {\scriptsize $\pm$ $0.01$} & $\mathbf{0.57}$ {\scriptsize $\pm$ $\mathbf{0.00}$} & $\mathbf{0.29}$ {\scriptsize $\pm$ $\mathbf{0.00}$} & $\mathbf{2.45}$ {\scriptsize $\pm$ $\mathbf{0.07}$} \\
    & Baldwinian & $1.60$ {\scriptsize $\pm$ $0.44$} & $11.07$ {\scriptsize $\pm$ $0.25$} & $0.53$ {\scriptsize $\pm$ $0.01$} & $2.19$ {\scriptsize $\pm$ $0.05$} & $1.74$ {\scriptsize $\pm$ $0.05$} & $17.97$ {\scriptsize $\pm$ $0.42$} \\
    & Lamarckian & $0.89$ {\scriptsize $\pm$ $0.42$} & $4.93$ {\scriptsize $\pm$ $0.09$} & $\mathbf{0.24}$ {\scriptsize $\pm$ $\mathbf{0.01}$} & $0.94$ {\scriptsize $\pm$ $0.02$} & $0.64$ {\scriptsize $\pm$ $0.02$} & $6.63$ {\scriptsize $\pm$ $0.24$} \\
    & LB & $0.76$ {\scriptsize $\pm$ $0.16$} & $5.02$ {\scriptsize $\pm$ $0.06$} & $0.25$ {\scriptsize $\pm$ $0.00$} & $1.10$ {\scriptsize $\pm$ $0.01$} & $0.64$ {\scriptsize $\pm$ $0.02$} & $8.96$ {\scriptsize $\pm$ $0.41$} \\
\bottomrule
\end{tabular}}
\end{table}
\FloatBarrier

\subsection{Cross-Problem Parameter Combination Selection and Universality Cost}
\label{app:universality}

We now detail the procedure behind the parameter universality cost reported in Table~\ref{tab:normalized_agg_loss}. For each evolution type, the goal is to select a universal parameter combination that performs well across both problems (MIS and MC) on all datasets, and to quantify the performance lost by using it in place of a specialised combination tuned per (problem, dataset) pair.

\textbf{Cross-Problem Parameter Combination Selection.} For each (problem, dataset, parameter combination), we compute a mean score by averaging over repetitions and graphs. These raw scores have scales that vary across datasets and problems, and therefore cannot be compared directly. We thus normalise them within each (problem, dataset, evolution type) group: every parameter combination's mean score is divided by the best mean score achieved for that evolution type on that (problem, dataset) pair. This places all values on a common $(0,1]$ scale, where $1$ marks the best-performing combination of that evolution type on that (problem, dataset) pair. For each evolution type, we then average these normalised scores across all (problem, dataset) pairs and select the highest-scoring combination as the universal one. The resulting combinations are listed in the caption of Table~\ref{tab:normalized_agg_loss}.

\textbf{Universality Cost.} For each (evolution type, problem, dataset) cell, we compare the mean score of the universal parameter combination, $s_{\text{universal}}$, against the score of the best-performing combination for that (problem, dataset) pair, $s_{\text{specialised}}$. The relative loss is
\[
\text{loss} = \frac{s_{\text{specialised}} - s_{\text{universal}}}{s_{\text{specialised}}} \times 100\%.
\]
Table~\ref{tab:normalized_agg_loss} reports this loss for each problem, dataset and evolution type, along with the per-problem averages.

\section{Formal Proof of Theorem \ref{thm:main}}
\label{app:proofs}

We now provide the full proof of Theorem \ref{thm:main}. Note that since we treat $k$ as a constant, in the following proof we do not optimise for the dependence of our runtime bounds on~$k$. In fact, our result is insightful already for $k=2$ and $k=3$, hence determining the asymptotics also in $k$ is less important.

\begin{proof}[of Theorem~\ref{thm:main}]
  We start with proving the upper bound for the Darwinian EA, where we can use a standard fitness level argument. Assume that the current solution of the EA has $c(x)-1$ leading all-ones blocks and that the critical block contains $\ell \in [1..k]$ zeros. Then with probability at least $n^{-\ell} (1 - 1/n)^{n-\ell} \ge \frac 1e n^{-k}$, exactly the zeros in the critical block are flipped and consequently the offspring contains at least $c(x)$ leading all-ones blocks. Hence the expected waiting time for increasing the number of leading all-ones blocks is $O(n^k)$. Note that such an offspring is accepted, that is, becomes the new parent, and note further, that the quantity $c(x)$ for the parent individual cannot reduce over time due to the elitist selection of the \oea. Consequently, after at most $n/k = O(n)$ such improvements, the EA has found the optimum. Adding the waiting times for the improvements gives the desired bound for the expected runtime of $O(n^{k+1})$. 

  For the lower bound, we use a recent result~\cite{DoerrK24level} that allows us to prove lower bounds via fitness level arguments and level visiting probabilities. We use the true fitness levels, that is, for each $i \in [0..n]$ the $i$-st fitness level consists of all solutions having fitness exactly~$i$. We are particularly interested in the case $i = jk+k-1$, that is, the solutions having $j$ leading all-ones blocks and then only zeros in the critical block. We first argue that each such fitness level is visited with probability at least $\frac{1}{2^k}$ (this is a very pessimistic estimate, but it suffices for our purposes). Fix $j \in [0..n/k-1]$. Let $x$ be the first solution generated by the algorithm that has at least $j$ leading all-ones blocks. If $x$ is the random initial solution, then, conditional on having at least $j$ leading all-ones blocks, the probability that $x_{jk+1}, \dots, x_{jk+k}$ are all zero is $2^{-k}$. In this case, we have $\DLB(x) = jk+k-1$. 
  Let now $x$ not be the initial solution. Then it was generated from some solution $y$ having less than $j$ leading all-ones blocks. A simple induction shows that in $y$ and in all solutions previously generated by the algorithm, the bits at positions $jk+1, \dots, jk+k$ are independently and uniformly distributed (since this is true for the initial solution and since also in the following, the values of these bits never had an influence on the run of the algorithm). This is a classic argument known from the analysis of \leadingones, see \cite[Lemma~1]{BottcherDN10} or \cite[Lemma~1]{LehreW12}, that has been extended to $\DLB_2$ in~\cite[Lemma~5]{WangZD24}. 
  Since $x$ is an offspring of such a $y$ and since all we condition on is that $x$ has the first $jk$ bits equal to one, also in $x$ the bits at positions $jk+1, \dots, jk+k$ are independently and uniformly distributed. Consequently, they are all zero with probability $2^{-k}$. We have thus shown that the first solution $x$ with at least $j$ leading all-ones blocks with probability $2^{-k}$ has fitness $jk+k-1$. Since any run of the algorithm at some time visits such a solution $x$, we know that it visits the $(jk+k-1)$-st fitness level with probability at least $v_{jk+k-1} := 2^{-k}$. When on this fitness level, the probability to leave it in one iteration is exactly $p_{jk+k-1} := (1-1/n)^{jk} n^{-k}$, since any solution with higher fitness has the first $jk+k$ bits equal to one. By~\cite[Theorem~8]{DoerrK24level}, the expected runtime of our algorithm is at least $\sum_{i=0}^{n-1} \frac{v_i}{p_i}$, where $v_i$ is a lower bound for the probability to visit the $i$-st fitness level and $p_i$ is an upper bound for the probability to leave it in one iteration. Using the values we just estimated and only regarding these fitness levels, we obtain the lower bound 
  \[
  \sum_{j=0}^{n/k-1} \frac{v_{jk+k-1}}{p_{jk+k-1}} \ge \sum_{j=0}^{n/k-1} 2^{-k} n^k = \Omega(n^{k+1}),
  \]
  since we treat $k$ as a constant. 

  We now continue with the proof for the Lamarckian \oea. By definition, all its parent solutions are local optima. By definition of the \DLB problem, these are exactly the bit strings with $jk$ leading ones followed by $k$ zeros, for some $j \in [0..n/k-1]$, together with the global optimum, the all-ones string. For the upper bound, we use a fitness level argument for these fitness levels. 
  Assume that the current solution has $jk$ leading ones followed by $k$ zeros, for some $j \in [0..n/k-1]$. Then with probability exactly $(1-1/n)^{jk+1} n^{-(k-1)} \ge \frac 1e n^{-k+1}$ the offspring has $jk + k-1$ leading ones followed by a zero. Now the local search procedure finds two types of improving neighbours, namely $k-1$ in which a one in the critical block is flipped to zero and the one in which the last zero in the critical block is flipped to one. The latter having more leading all-ones blocks than the former, it also has the highest fitness, hence this is the next solution visited by the local search procedure. Since the local search procedure does not accept worsenings, these $j+1$ leading all-ones blocks will not be touched anymore. Consequently, we have obtained a solution with fitness better than the parent, that is, we have left the fitness level. We have thus shown that each of the $n/k$ non-optimal fitness levels are left with probability at least $\frac 1e n^{-k+1}$ per iteration. By the classic fitness level theorem~\cite{Wegener05}, this gives an upper bound of $(n/k) en^{k-1} = O(n^k)$ iterations for the expected runtime.

  For the lower bound, we again use the fitness level method for lower bounds, with some extra care with respect to the local search step. We estimate the visiting probabilities for the fitness levels. Recall that for the Lamarckian algorithm only solutions with fitness $kj + k - 1$, that is, having $j$ leading all-ones blocks and then $k$ zeros, will be taken as parent individual.   
  Let $x$ be the first individual forming the parent population that has fitness at least $kj + k - 1$. This individual is either the random initial individual $y$ to which local search was applied, or a mutation offspring $y$ to which local search was applied. Since all previous individuals had a lower fitness (also after applying local search), as in the lower-bound proof for Darwinian evolution, the bits $y_{kj+1}, \dots, y_{kj+k}$ are uniformly and independently distributed. In particular, with probability $1 - (k+1)2^{-k}$, at least two of them are zero. 
  Consequently, in this case, the local search procedure applied to $y$ will, after having possibly optimised earlier blocks to all-ones blocks, optimise the $(j+1)$-st block to an all-zero block. 
  This shows that the Lamarckian EA visits the fitness level $kj + k - 1$ with probability $v_{kj + k - 1} = 1 - (k+1)2^{-k} \ge 1/4$.   
  When on such a fitness level, that is, when the parent individual has $j$ leading all-ones blocks followed by an all-zeros block, to leave this fitness level to a higher level it is necessary that the mutation flips at least $k-1$ of the zeros in the critical block (otherwise, the local search cannot transform the critical block into an all-ones block). This happens with probability $kn^{-k+1}+n^{-k} \le (k+1)n^{-k+1} =: p_{kj + k - 1}$. Using the fitness-level theorem for lower bounds again, we obtain a lower bound for the expected runtime of $\sum_{j=0}^{n/k-1} \frac{v_{kj + k - 1}}{p_{kj + k - 1}} \ge (n/k)(1/4)n^{k-1}/(k+1) = \frac{n^k}{12(k+1)} = \Omega(n^k)$. 

  We finally turn to the Baldwinian EA. Since the fitness of any individual is the fitness of a local optimum, the same fitness levels are visited as for the Lamarckian EA, but a wider set of parent individuals may occur. More specifically, if the Baldwinian fitness of an individual is $jk+k-1$, then the individual can be any that has at least $k-1$ ones in each of the first $j$ blocks and at most $k-2$ ones in the $(j+1)$-st block (and these are all possibilities). 

  On this fitness level, the algorithm performs a Markov chain as follows. Whether the offspring is accepted is determined by the first $j$ blocks, namely whether the offspring has again at least $k-1$ ones in each of these blocks. This happens with probability at least $(1-1/n)^{kj} \ge \frac 1e$, simply by regarding the event that none of the first $kj$ bits is flipped. In the $(j+1)$-st block, each bit is flipped independently with probability $1/n$. By a simple union bound, the probability that some bit there is flipped is at most $k/n = \Theta(1/n)$, hence the typical event is that no bit is flipped. Again by a union bound over the sets of two bits in this block, the probability that more than one bit is flipped in this block is at most $\binom{k}{2} n^{-2} = \Theta(1/n^2)$. Hence, very roughly speaking, what happens in the $(j+1)$-st block is a pure random walk, slowed down by a factor of $\Theta(n/k)$, with very exceptional events that we move by more than one step. This almost-random walk ends when at least $k-1$ ones in the critical block are reached, as then the local search will bring us to a higher fitness level. All this allows us to use a runtime analysis of the randomized local search (RLS) heuristic on generalized needle functions~\cite{DoerrK25}, as follows. Consider a sequence of 
  \[
  t = 16 \cdot 2^k n
  \]
  iterations. The probability that in this time interval at least once two bits in the $(j+1)$-st block are flipped is at most $t \binom{k}{2} n^{-2} = O(1/n)$. Let us call an iteration a \emph{move} if the actions of the mutation operator are such that the offspring is accepted (as discussed, this is decided in the first $j$ blocks) and that exactly one bit in the $(j+1)$-st block is flipped. The probability that in an iteration none of the first $kj$ bits flips and exactly one in the $(j+1)$-st block is exactly $k (1/n) (1-1/n)^{kj+k-1} \ge \frac{k}{en}$. Hence this is a lower bound for the probability of a move. The expected number of moves in $t$ iterations therefore is at least $\frac{tk}{en} = \frac{16 \cdot 2^k k}{e} =: E$. By a multiplicative Chernoff bound together with an argument that the expectation can be replaced by a lower bound for it (e.g., Theorem~1.10.5 and Section 1.10.1.8 in~\cite{Doerr20bookchapter}), we obtain that the probability to have less than $E/2$ moves in these $t$ iterations is at most $\exp(-E/8) \le \exp(-8/e)$. Note that a move means that in the $(j+1)$-st block, we do a RLS step. By Theorem~1 of~\cite{DoerrK25}, the expected number of RLS steps it takes to reach at least $k-1$ ones, regardless of the starting point, is at most 
  \[
  \sum_{\ell=0}^{k-2} \binom{k}{\le \ell} \big/ \binom{k-1}{\ell} \le k 2^k / (k-1) \le 2^{k+1}, 
  \]
  where $\binom{k}{\le \ell} := \sum_{i=0}^{\ell} \binom{k}{i}$. Hence the probability that $E/2$ moves do not suffice, by Markov's inequality, is at most $2^{k+1} / (E/2) = e/4$. Adding up the failure probabilities, we see that the probability that $t$ iterations do not let us leave the fitness level is at most $O(1/n) + \exp(-8/e) + e/4$, which is at most $3/4$ when assuming $n$ to be large enough (which we may since we only prove an asymptotic bound). This proves that any interval of $t$ iterations lets us leave the current fitness level with probability at least $1/4$. This gives an expected time of at most $4t$ for leaving one level, and hence an expected runtime of $4t (n/k) = O(n^2)$.

  For the lower bound, we use again the fitness level method with visiting probabilities. The probability $v_{kj + k - 1}$ to visit a fitness level is the same as for the Lamarckian EA, since all arguments given there apply here as well. We estimate the probability $p_{kj + k - 1}$ to leave a level by $p_{kj + k - 1} \le k/n$, simply because to leave a level it is necessary to flip at least one bit in the critical block. From these $p_{kj + k - 1} =O(1/n)$ and $v_{kj + k - 1} = \Theta(1)$, noting that we have $\Omega(n)$ levels, we obtain the desired lower bound of $\Omega(n^2)$.\qed 
\end{proof}
\FloatBarrier

}
\end{document}